\documentclass[onefignum,onetabnum]{siamonline190516}
\graphicspath{{figures/}{../figures}}
\usepackage{xcolor,soul,amssymb,amsmath,dsfont} 
\usepackage{tikz} 
\usetikzlibrary{matrix,calc}
\usepackage{color}
\usepackage{graphicx,algorithm,algorithmic}
\graphicspath{{figures/}{./}}
\usepackage{array}
\usepackage{eqparbox}
\usepackage{url}
\usepackage{relsize}
\usepackage{rotating}
\usepackage{subcaption}
\usepackage{lipsum}
\usepackage{amsfonts}
\usepackage{graphicx}
\usepackage{epstopdf}
\usepackage{algorithmic}
\ifpdf
  \DeclareGraphicsExtensions{.eps,.pdf,.png,.jpg}
\else
  \DeclareGraphicsExtensions{.eps}
\fi

\DeclareMathOperator*{\argmax}{arg\,max}

\newcommand{\red}{\color{red}}

\usepackage{enumitem}
\setlist[enumerate]{leftmargin=.5in}
\setlist[itemize]{leftmargin=.5in}

\newsiamremark{remark}{Remark}
\newsiamremark{problem}{Problem}
\newsiamremark{assumption}{Assumption}
\newsiamremark{example}{Example}

\headers{Sign-indefinite Priors \& a Sinkhorn-type algorithm}{A. Dong, T. T. Georgiou, and A. Tannenbaum}

\title{Data Assimilation for Sign-indefinite Priors:\\
    A generalization of Sinkhorn's algorithm\thanks{Submitted to the editors in August 2023.\funding{Research supported in part by the  AFOSR under grant FA9550-23-1-0096 and ARO under W911NF-22-1-0292. The authorship is alphabetical.}}}

\author{Anqi Dong\thanks{Department of Mechanical and Aerospace Engineering, University of California, Irvine, CA 92697, USA
(\email{anqid2@uci.edu},\email{tryphon@uci.edu}).}
\and Tryphon T. Georgiou\footnotemark[2]\textsuperscript{ ,}\thanks{Corresponding author}
\and Allen Tannenbaum\thanks{Departments of Computer Science and Applied Mathematics \& Statistics at the State University of New York, Stony Brook, NY 11794, USA;  (\email{allen.tannenbaum@stonybrook.edu}).}}

\usepackage{amsopn}

\ifpdf
\hypersetup{
  pdftitle={Data assimilation for sign-indefinite priors:}{A generalization of Sinkhorn's algorithm},
  pdfauthor={A. Dong, T. T. Georgiou, and A. Tannenbaum}
}
\fi

\begin{document}

\maketitle

\begin{abstract}
The purpose of this work is to develop a framework to calibrate signed datasets so as to be consistent with specified marginals by suitably extending the Schr\"odinger-Fortet-Sinkhorn paradigm. Specifically, we seek to revise sign-indefinite multi-dimensional arrays in a way that the updated values agree with specified marginals. Our approach follows the rationale in Schr\"odinger's problem \cite{schrodinger1931umkehrung,schrodinger1932theorie}, aimed at updating a ``prior'' probability measure to agree with marginal distributions. The celebrated Sinkhorn's algorithm \cite{sinkhorn1964relationship,knight2008sinkhorn} (established earlier by R.\ Fortet \cite{fortet1940resolution,essid2019traversing}) that solves Schr\"odinger's problem found early applications in calibrating contingency tables in statistics and, more recently, multi-marginal problems in machine learning and optimal transport. 
Herein, we postulate a sign-indefinite prior in the form of a multi-dimensional array, and propose an optimization problem to suitably update this prior to ensure consistency with given marginals. The resulting algorithm generalizes the Sinkhorn algorithm in that it amounts to iterative scaling of the entries of the array along different coordinate directions. The scaling is multiplicative but also, in contrast to Sinkhorn, inverse-multiplicative depending on the sign of the entries. Our algorithm reduces to the classical Sinkhorn algorithm when the entries of the prior are positive.
\footnote{Research supported in part by the  AFOSR under grant FA9550-23-1-0096 and ARO under W911NF-22-1-0292. The authorship is alphabetical.}
\end{abstract}


\begin{AMS}
  {15Axx, 45G15, 49M05, 49-xx, 60-xx}
\end{AMS}

\section{Introduction}\label{sec:intro}
Consider a multi-dimensional array $Q_{i_1i_2\cdots i_n}$, with entries that are sign indefinite, in general, and indices taking values on finite intervals $\mathcal I_\ell\subset \mathbb N$ of integers, so that $i_\ell\in\mathcal I_\ell$.
The problem we consider is to revise the given values in $Q$, viewed as a ``prior,'' so as to obtain a ``posterior'' $P$ with entries $P_{i_1i_2\cdots i_n}$, that is now in agreement with specified {\em positive} marginals
\begin{subequations}\label{eq:data}
\begin{align}\label{eq:data1}
    p_{i_\ell}^{(\ell)}&=\sum_{i_1i_2 \cdots i_n \backslash i_\ell}P_{i_1i_2\cdots i_n},
\end{align}
for $i_\ell\in\mathcal I_\ell$ and $\ell\in\{1,\ldots,n\}$. The notation ${i_1i_2 \cdots i_n \backslash i_\ell}$ denotes the collection of all indices except $i_\ell$, and when used under a summation as above, denotes summation over all the values of the collection of the indices that are included.
Alternatively, we may postulate a prior non-negative (probability) measure $\mathbf Q_{i_1i_2\cdots i_n}$, with similar range of indices, and seek a nonnegative posterior $\mathbf P_{i_1i_2\cdots i_n}$, so that this is in agreement with expectations
\begin{align}\label{eq:data2}
    p_{i_\ell}^{(\ell)}&=\sum_{i_1i_2 \cdots i_n \backslash i_\ell}X_{i_1i_2\cdots i_n}\mathbf P_{i_1i_2\cdots i_n},
\end{align}
\end{subequations}
that are positive. In this case, $X_{i_1i_2\cdots i_n}$ is a random variable, i.e., a function on the set of allowed indices, taking values in $\{-1,+1\}$, specifically\footnote{We follow the convention that ${\rm sign}(0)=1$.} $X_{i_1i_2\cdots i_n}={\rm sign}(Q_{i_1i_2\cdots i_n})$. Either formulation seeks to reconcile the prior, by adjusting accordingly its values in obtaining the posterior, so as to match signed (one bit) local averages.

Although this may be seem as a special case of general type of moment problems, it has a special structure that we wish to explore. Specifically, postulating that the posterior is close to the prior in a relative entropy sense, i.e.,
as a minimizer of
\begin{subequations}\label{eq:relativeentropy}
\begin{align}\label{eq:relativeentropy1}
    S(P,Q):=\sum_{i_1i_2 \cdots i_n }|P_{i_1i_2\cdots i_n}|\left(\log\left(\frac{P_{i_1i_2\cdots i_n}}{Q_{i_1i_2\cdots i_n}}\right)-1\right)
\end{align}
setting $P$ to inherit the sign of $Q$ at the same indices,
we derive a generalized Sinkhorn-like algorithm. Equivalently, we may consider minimizing
\begin{align}\label{eq:relativeentropy2}
    S(\mathbf P,\mathbf Q):=\sum_{i_1i_2 \cdots i_n }\mathbf P_{i_1i_2\cdots i_n}\left(\log\left(\frac{\mathbf P_{i_1i_2\cdots i_n}}{\mathbf Q_{i_1i_2\cdots i_n}}\right)-1\right)
\end{align}
subject to \eqref{eq:data2}.
\end{subequations}
The generalized Sinkhorn algorithm amounts to circularly scaling entries along various coordinate directions so that the respective sums match the given averages \eqref{eq:data1}, continuing until convergence. In the classical setting of the Sinkhorn algorithm, the scaling is the same for all elements along specified coordinate directions whereas in the case of the generalized Sinkhorn algorithm herein, the scaling differentiates positive and negative entries and, respectively, scales by a factor or its inverse.

The original inception of the Sinkhorn scheme was in the context of the Schr\"odinger problem \cite{schrodinger1931umkehrung,schrodinger1932theorie}, originally cast in a dynamic setting, to update a prior law of diffusive particles so as to agree with specified marginal distributions \cite{chen2021stochastic,chen2015optimal}. Existence and uniqueness of solution, and in fact, the iterative algorithm for solving the so-called Schr\"odinger-system was provided by the great French mathematician R.\ Fortet \cite{fortet1940resolution,essid2019traversing}. This same iterative algorithm, in the special case of finite-dimensions, came to be known as the Sinkhorn, or as the Sinkhorn-Knopp algorithm \cite{sinkhorn1964relationship,knight2008sinkhorn}. The roots of the Sinkhorn algorithm go back to the problem of calibrating contingency tables in dawn of statistics. In the past ten years, the Sinkhorn algorithm has become the {\em workhorse} of computational optimal mass transport and the enabling tool for many problems in machine learning \cite{peyre2019computational,chen2021stochastic,chen2016entropic};  for recent developments we refer to \cite{carlier2022linear} regarding multi-marginal problems, \cite{haasler2021multi} regarding graphical models, and  \cite{georgiou2015positive,friedland2017schrodinger} on extensions to the setting of quantum probabilities.

The motivation and inception of the present work, in introducing a Schr\"odinger-type problem with sign-indefinite prior, originates in the problem to estimate parameters of gene regulatory networks from protein-expression levels \cite{dong2023negative} (see also \cite{sandhu2015graph,sandhu2016geometry,seccilmics2020uncovering} for further background and on-going developments). The salient feature of this setting is that edge-weights, in gene-regulatory networks, may model suppression/promotion in co-expression of genes --this feature brings in naturally a sign differentiation of entries to reflect respective contributions of nodes in protein level production.
Biology and chemistry are seen to provide a potentially rich class of applications, other subjects such as financial mathematics and operations research have presented applications that call for indefinite priors \cite{burgin2012negative,cui2010reliable}. The present paper focuses on the formulation of the pertinent mathematical problem.

In the body of the paper, we first discuss in Section \ref{sec:sinkhorn} the Sinkhorn algorithm from an angle that facilitates comparison and links to the generalization that follows. Then, in Section \ref{sec:generalized}, we develop the generalized Sinkhorn algorithm that applies to the case of sign-indefinite prior. We present illustrative examples in Section \ref{sec:num} and concluding remarks in Section \ref{sec:conclusion}.

\section{The Sinkhorn iteration}\label{sec:sinkhorn}

We briefly review the Sinkhorn iteration as it applies to the case where the entries of $P,Q$ are positive\footnote{The case where the entries are non-negative is similar, however, in this case, an additional technical condition is needed to guarantee convergence.}. The minimizer of \eqref{eq:relativeentropy1} subject to the constraints \eqref{eq:data1} is a stationary point of the Lagrangian
\begin{align}\label{eq:Lagrangian}
     &\mathcal L(P_{i_1i_2\cdots i_n},\lambda^{(1)}_{i_1},\ldots,\lambda^{(n)}_{i_n})\\ \nonumber
     & \phantom{xxxxxxx} 
     = \sum_{i_1i_2 \cdots i_n} P_{i_1i_2\cdot i_n}
     \bigg(\log\Big(\frac{P_{i_1i_2\cdots i_n}}{Q_{i_1i_2\cdots i_n}}\Big)-1\bigg) 
     + 
     \sum_{\ell=1}^n \lambda^{(\ell)}_{i_\ell} \bigg(\sum_{i_1i_2 \cdots i_n/i_\ell} \!\!\! P_{i_1i_2\cdots i_n} - p_{i_\ell}^{(\ell)} \bigg).
\end{align}
Thus, setting partial derivatives with respect to $P_{i_1i_2\cdots i_n}$ to zero,
the minimizer must be
\begin{align}\label{eq:eqopt}
    P_{i_1i_2\cdots i_n}^\star &= Q_{i_1i_2\cdots i_n}\exp\left(-\sum_{\ell=1}^n\lambda^{(\ell)}_{i_\ell}\right)\\
    &= Q_{i_1i_2\cdots i_n}\prod_{\ell=1}^n a_{i_\ell}^{(\ell)},\nonumber
\end{align}
setting $a_{i_\ell}^{(\ell)}:=e^{-\lambda^{(\ell)}_{i_\ell}}$.
The Sinkhorn algorithm adjusts successively the $a_{i_\ell}^{(\ell)}$'s
so as to satisfy the constraints \eqref{eq:data1}, and continues circularly until convergence. That is, writing the constraint
\begin{align*}
    p_{i_\ell}^{(\ell)}&= \!\!\! \sum_{i_1i_2 \cdots i_n \backslash i_\ell} \!\!\!Q_{i_1i_2\cdots i_n}\!\prod_{\ell=1}^n a_{i_\ell}^{(\ell)}
    =\left(\sum_{i_1i_2 \cdots i_n \backslash i_\ell}\!\!\!\bigg(Q_{i_1i_2\cdots i_n}\!\!\!\prod_{i_1i_2 \cdots i_n \backslash i_\ell} \!\!\! a_{i_k}^{(k)}\bigg)\right)a_{i_\ell}^{(\ell)},
\end{align*}
define
\[
  S_{i_\ell}(a_{i_1i_2 \cdots i_n \backslash i_\ell})=\!\!\!\!\sum_{i_1i_2 \cdots i_n \backslash i_\ell}\!\! \left(Q_{i_1i_2\cdots i_n}\!\!\prod_{i_1i_2 \cdots i_n \backslash i_\ell} \!\!\!\! a_{i_k}^{(k)}\right),
\]
and the Sinkhorn ratio
\[
  \mathbb S(S_{i_\ell}(a_{i_1i_2 \cdots i_n \backslash i_\ell}),p_{i_\ell}^{(\ell)}):=p_{i_\ell}^{(\ell)}/S_{i_\ell}(a_{i_1i_2 \cdots i_n \backslash i_\ell}).
\]
The Sinkhorn iteration proceeds circularly as follows.
\begin{algorithm}[H]
\caption{Sinkhorn algorithm}\label{alg:1}
\begin{algorithmic}[1]
\STATE Initialize $a_{i_\ell}^{(\ell)}=1$, for $\ell\in\{1,\ldots,n\}$ and $i_\ell\in\mathcal I_\ell$.\\[0.045in]
\STATE
For $\ell=1:n$ and $i_\ell\in\mathcal I_\ell$, update the value 
of $a_{i_\ell}^{(\ell)}$ to meet the corresponding constraint,
by setting
\[
a_{i_\ell}^{(\ell)}=\mathbb S(S_{i_\ell}(a_{i_1i_2 \cdots i_n \backslash i_\ell}),p_{i_\ell}^{(\ell)}).    
\]
\STATE Repeat step 2 until convergence.
\STATE Obtain $P_{i_1i_2\cdots i_n}^\star=Q_{i_1i_2\cdots i_n}\prod_{\ell=1}^n a_{i_\ell}^{(\ell)}$.
\end{algorithmic}
\end{algorithm}

When $Q$ is a matrix with only two indices, the Sinkhorn algorithm alternates between scaling rows and columns, a process referred to as diagonal scaling. The convergence and performance of the Sinkhorn algorithm have been studied extensively; for the multi-marginal setting we refer to~\cite{carlier2022linear}. The algorithm can be seen as a {\em coordinate ascent} method to maximize the dual functional
\begin{align*} 
g_s(\lambda^{(1)}_{i_1},\ldots,\lambda^{(n)}_{i_n})
=\sum_{i_1i_2 \cdots i_n }\left( -Q_{i_1i_2\cdots i_n}\exp\left(-\sum_{\ell=1}^n\lambda^{(\ell)}_{i_\ell}\right) -\lambda^{(\ell)}_{i_\ell}p^\ell_{i_\ell}\right),
\end{align*}
obtained by substituting $P_{i_1i_2\cdots i_n}^\star$ into the Lagrangian~\eqref{eq:Lagrangian}.
Indeed,
\begin{align*}
    \lambda^{(\ell)}_{i_\ell} = \argmax_{\lambda^{(\ell)}_{i_\ell}} \, g_s(\lambda^{(1)}_{i_1},\ldots,\lambda^{(n)}_{i_n}),
\end{align*}
for $i_\ell\in\mathcal I_\ell$, gives $a_{i_\ell}^{(\ell)}:=e^{-\lambda^{(\ell)}_{i_\ell}}$ as in Algorithm~\ref{alg:1}. Interestingly, a result by Luo and Tseng~\cite{luo1992convergence}, \cite{wright2015coordinate} gives immediately that the algorithm being coordinate ascent has a linear convergence rate, assuming feasibility.

\section{The generalized Sinkhorn algorithm}\label{sec:generalized}
We now return to the case of multi-index arrays $Q,~P$ with sign-indefinite entries. The Lagrangian in this case, e.g., $\mathcal L(\mathbf P_{i_1i_2\cdots i_n},\lambda^{(1)}_{i_1},\ldots,\lambda^{(n)}_{i_n})$ for \eqref{eq:relativeentropy2}, is
\begin{align*}
     &\sum_{i_1i_2 \cdots i_n }
     \mathbf P_{i_1i_2\cdots i_n}\left(\log\left(\frac{\mathbf P_{i_1i_2\cdots i_n}}{\mathbf Q_{i_1i_2\cdots i_n}}\right)-1\right)
     + \sum_{\ell=1}^n\lambda^{(\ell)}_{i_\ell} \bigg(\sum_{i_1i_2 \cdots i_n/i_\ell}\!\!\!\! X_{i_1i_2\cdots i_n}\mathbf P_{i_1i_2\cdots i_n} - p_{i_\ell}^{(\ell)} \bigg),
\end{align*}
where $X_{i_1i_2\cdots i_n}:={\rm sign}(Q_{i_1i_2\cdots i_n})$ and $\mathbf P,\mathbf Q$ as explained earlier. We assume throughout the following.

\begin{assumption}\label{slater} The intersection
\begin{align*}
    \Big\{\mathbf P~| \sum_{i_1i_2 \cdots i_n \backslash i_\ell}\!\!\! X_{i_1i_2\cdots i_n}\mathbf P_{i_1i_2\cdots i_n} = p_{i_\ell}^{(\ell)},\ \forall~i_\ell\in\mathcal I_\ell, 1\leq\ell\leq n\Big\}\cap
    \Big\{\mathbf P~| S(\mathbf P,\mathbf Q) < +\infty \Big\},
\end{align*}
has a non-empty interior.
\end{assumption}

This, of course, ensures a unique minimizer due to the strict convexity of the entropy functional. Feasibility of related multivariable optimal transport problems is discussed by Rachev and R\"uschendorf~\cite{rachev1998mass}. The complexity of testing feasibility of the present ``signed transport problem'' is not known, and for this, convergence of the algorithm that follows may be the simplest option.

Setting the partial derivatives of the Lagrangian to zero gives the minimizer in the form
\begin{align}\nonumber
   \mathbf P^\star_{i_1i_2\cdots i_n} &= \mathbf Q_{i_1i_2\cdots i_n}\times\exp\bigg(- X_{i_1i_2\dots i_n} \Big(\lambda_{i_1}+ \dots + \lambda_{i_n}\Big) \bigg)\\\label{eq:P2}
   &=\mathbf Q_{i_1i_2\cdots i_n} \prod_{\ell=1}^n \alpha_{i_\ell}^{X_{i_1i_2\dots i_n}},
\end{align}
having set $\alpha_{i_\ell}=e^{-\lambda_{i_\ell}}.$ Thus, written separately in the two cases, depending on the sign of $X$,
\begin{align*}
\mathbf P^*_{i_1i_2\cdots i_n}  &= \mathbf Q_{i_1i_2\cdots i_n}\prod_{\ell=1}^n \alpha_{i_\ell},   \mbox{ when }X_{i_1i_2\dots i_n} > 0, \mbox{ and}\\
& = \mathbf Q_{i_1i_2\cdots i_n}\prod_{\ell=1}^n\alpha_{i_\ell}^{-1}. \mbox{ when }X_{i_1i_2\dots i_n} < 0.
\end{align*}

We now express the conditions for optimality directly in terms of $Q$ and derive a generalization of the Sinkhorn algorithm. To this end we define $\mathbf Q:=|Q|$, setting $\mathbf Q_{ij}=|Q_{ij}|$, and we also define the positive and negative parts of $Q$ as follows, $Q^+:=\frac12(\mathbf Q + Q)$ and  $Q^-:=\frac12(\mathbf Q - Q)$. Hence, $\mathbf Q=Q^++Q^-$ while $Q=Q^+-Q^-$.
The optimal $P^\star$ can be written as
\begin{align*}
P^\star_{i_1i_2\cdots i_n} =&  \left(( Q^{+}_{i_1i_2\cdots i_n} \prod_{\ell=1}^n \alpha_{i_\ell})- (Q_{i_1i_2\cdots i_n}^- \prod_{\ell=1}^n \alpha_{i_\ell}^{-1})\right)\\
=& \left((Q^+_{i_1i_2\cdots i_n} \prod_{i_1i_2 \cdots i_n \backslash i_\ell} \!\!\!\!\! \alpha_{i})\ \alpha_{i_\ell}-(Q^-_{i_1i_2\cdots i_n} \prod_{i_1i_2 \cdots i_n \backslash i_\ell} \!\!\!\!\!\alpha_{i}^{-1}) \ \alpha_{i_\ell}^{-1}\right),
\end{align*}
with constraints 
\begin{align}\tag{\ref{eq:data1}}
  &\sum_{i_1i_2 \cdots i_n \backslash i_\ell} \!\!\!\! P^\star_{i_1i_2\cdots i_n}=p_{i_\ell}^{\ell}.
\end{align}
We express explicitly the dependence of the above summation on $\alpha_{i_\ell}$, as follows. Define
\begin{align*}
    &S^+_{i_\ell}(\alpha_{i_1i_2 \cdots i_n \backslash i_\ell}):= \!\!\!\! \sum_{i_1i_2 \cdots i_n \backslash i_\ell} \!\!\! (Q^+_{i_1i_2\cdots i_n} \!\! \prod_{i_1i_2 \cdots i_n \backslash i_\ell} \!\!\!\! \alpha_{i}),\\
    &S^-_{i_\ell}(\alpha_{i_1i_2 \cdots i_n \backslash i_\ell}):=  \!\!\!\! \sum_{i_1i_2 \cdots i_n \backslash i_\ell} \!\!\! (Q^-_{i_1i_2\cdots i_n} \!\! \prod_{i_1i_2 \cdots i_n \backslash i_\ell} \!\!\!\! \alpha_{i}^{-1}),
\end{align*}
and write \eqref{eq:data1} in the form
\begin{align}\label{eq:aainvp}
    &S^+_{i_\ell}(\lambda_{i_1i_2 \cdots i_n \backslash i_\ell})\ \alpha_{i_\ell}-S^-_{i_\ell}(\lambda_{i_1i_2 \cdots i_n \backslash i_\ell})\ \alpha_{i_\ell}^{-1}=p_{i_\ell}^{\ell},
\end{align}
where $S^+,S^-,\alpha_{i_\ell},p_{i_\ell}^\ell$ are all positive. This equation can be readily solved for $\alpha_\ell$, which is set to be the positive root of the quadratic polynomial
\begin{equation}\label{eq:poly}
   S^+_{i_\ell}(\lambda_{i_1i_2 \cdots i_n \backslash i_\ell})\alpha_{i_\ell}^2-p_{i_\ell}^{\ell}\alpha_{i_\ell}-S^-_{i_\ell}(\lambda_{i_1i_2 \cdots i_n \backslash i_\ell})=0,
\end{equation}
giving,
\begin{align}\label{eq:logquad1}
   \alpha_{i_\ell}=\frac{ p_{i_\ell}^{\ell}+\sqrt{(p_{i_\ell}^{\ell})^2+4S^+_{i_\ell}S^-_{i_\ell}}}{2S^+_{i_\ell}}=:\mathbb G(S^+_{i_\ell},S^-_{i_\ell},p_{i_\ell}^{\ell}).
\end{align}
The quadratic polynomial in $\alpha_{i_\ell}$, in \eqref{eq:poly}, has a positive and a negative root with only the positive root being admissible. Bringing all of the above together, we arrive at the following iterative algorithm:
\begin{algorithm}[H]
\caption{Generalized Sinkhorn Algorithm}\label{alg:gen-sinkhorn}
\begin{algorithmic}[1]
\STATE Initialize $\alpha_{i_\ell}^{(\ell)}=1$, for $\ell\in\{1,\ldots,n\}$ and $i_\ell\in\mathcal I_\ell$.\\[0.045in]
\STATE 
For $\ell=1:n$ and $i_\ell\in\mathcal I_\ell$, set
\[
\alpha_{i_\ell}=\mathbb G(S^+_{i_\ell},S^-_{i_\ell},p_{i_\ell}^{\ell}).
\]
\STATE Repeat steps 2 until convergence.
\STATE Obtain $P_{i_1i_2\cdots i_n}^\star=Q_{i_1i_2\cdots i_n}\prod_{\ell=1}^n a_{i_\ell}^{(\ell)}$.
\end{algorithmic}
\end{algorithm}

The dual function of~\eqref{eq:relativeentropy2} is obtained by substituting the optimizer in~\eqref{eq:P2} into the Lagrangian. It reads
\begin{align*}
    g(\lambda^{(1)}_{i_1},\ldots,\lambda^{(n)}_{i_n}) 
    & =\sum_{i_1i_2 \cdots i_n }\left( -\mathbf P^\star_{i_1i_2\cdots i_n} -\lambda^{(\ell)}_{i_\ell}p^\ell_{i_\ell}\right),
\end{align*}
and the dual problem is to maximize a concave function,
\begin{align}\label{prob:dual}
    \max_{\lambda^{(1)}_{i_1},\ldots,\lambda^{(n)}_{i_n}}\ g(\lambda^{(1)}_{i_1},\ldots,\lambda^{(n)}_{i_n}).
\end{align}
Due to our standing Assumption \ref{slater} that Slater's condition holds~\cite{boyd2004convex,borwein2006convex} and the convexity of the primal, we have the following.

\begin{proposition}\label{prop:duality}
Strong duality holds between the primal problem to minimize \eqref{eq:relativeentropy2} subject to \eqref{eq:data2} and the dual \eqref{prob:dual}.    
\end{proposition}

We are now in a position to discuss convergence of Algorithm \ref{alg:gen-sinkhorn}. The argument is completely analogous to the earlier argument on the standard Sinkhorn, as it relies on a gradient ascent nature of the problem aforementioned the result of Luo and Tseng~\cite{luo1992convergence}.

\begin{theorem}\label{thm:convergence}
Algorithm~\ref{alg:gen-sinkhorn} converges with a linear convergence rate.   
\end{theorem}

\begin{proof}
The objective function in~\eqref{eq:relativeentropy2} is strictly convex, and thus the optimizer $\mathbf P^\star$ can be obtained using a %block 
coordinate ascent strategy with strong duality holding (Proposition~\ref{prop:duality}). Specifically, coordinate-wise optimization of the dual~\eqref{prob:dual}
\begin{align*}
    \lambda^{(\ell)}_{i_\ell} = \argmax_{\lambda^{(\ell)}_{i_\ell}}\ g(\lambda^{(1)}_{i_1},\ldots,\lambda^{(n)}_{i_n}), \ \forall i_\ell\in\mathcal I_\ell
\end{align*}
gives that
\begin{align*}
    \frac{\partial g(\lambda^{(1)}_{i_1},\ldots,\lambda^{(n)}_{i_n})}{\partial \lambda_{i_\ell}} = \!\!\!\!
    \sum_{i_1i_2 \cdots i_n \backslash i_1} \!\!\!\! X_{i_1i_2\dots i_n}\mathbf Q_{i_1i_2\cdots i_n} \prod_{\ell=1}^n \alpha_{i_\ell}^{X_{i_1i_2\dots i_n}}  - p^{\ell}_{i_\ell}=0
\end{align*}
at each step, leading to an update for $\alpha_{i_\ell}$ to satisfy \eqref{eq:aainvp}, and thereby, \eqref{eq:poly}.
The generalized Sinkhorn iteration in Algorithm~\ref{alg:gen-sinkhorn} inherits the linear convergence rate of %block 
coordinate ascent~\cite[Theorem 2.1, Application 5.3]{luo1992convergence}, as in the standard Sinkhorn. 
\end{proof}

\section{Numerical examples}\label{sec:num}
We elucidate the framework with $2$-marginal and $3$-marginal examples. For better visualization, we color negative entries in red and positive ones in blue, while respective marginals are marked in cyan.
The absolute value or intensities corresponds to the radius of circles and spheres, drawn in the $2-$ and $3$-dimensional cases, respectively. The code that was used for these numerical examples has been posted at~\url{https://github.com/dytroshut/Generalized-Sinkhorn-Algorithm}.

\subsection{Two-marginal case}
A prior $Q_{i_1i_2}$ is selected and shown in the left subplot of Figure~\ref{fig:2d}. Its entries have been assigned values in $\{0,\pm 1\}$. Positive values are drawn in blue and negative in red, as noted, and for ease of reference the ``prior'' marginals $q^{1}_{i_1} = \sum_{i_2}Q_{i_1i_2}$ and $q^{2}_{i_2} = \sum_{i_1}Q_{i_1i_2}$ are also depicted in the subplot and drawn by circles in cyan.
Further, we specify
(positive) marginals $p^{1}_{i_1}$ and $p^{2}_{i_2}$ with values
\begin{align*}
    p^{1}_{i_1} &= [0.1,0.05,0.05,0.15,0.2,0.05,0.03,0.07,0.25,0.05]\\
    p^{2}_{i_2} &= [0.05,0.1,0.05,0.2,0.07,0.15,0.05,0.25,0.03,0.05].
\end{align*}
Applying Algorithm \ref{alg:gen-sinkhorn} we obtain
\begin{equation}\label{eq:exampleP} 
P^\star = 
   {\footnotesize
   \begin{bmatrix}
   \phantom{-}{ 0.6140}  &{\red -0.2138}   &{ 0.0208}  &0   &{ 0.0461}  &0  &{\red-0.4165}  &0   &{ 0.0251}  &{0.0242}\\
   {\red-0.2616}   &\phantom{-}0  &{0.0176}  &0   &0   &0   &\phantom{-}{0.2737}  &0    &0    &{0.0203}\\
   \phantom{-}{ 0.0147}    &\phantom{-}{ 0.0152}   &0   &0    &{ 0.0011}   &0    &\phantom{-}{ 0.0078}  &{ 0.0112}  &0  &0\\
   \phantom{-}0   &\phantom{-}0   &0   &0    &{ 0.0135}   &0   &\phantom{-}0    &{ 0.1365}   &0    &0\\
   \phantom{-}{ 0.0463}   &\phantom{-}0   &{ 0.0016}  &{ 0.1327}  &{ 0.0035}  &{ 0.0160} &\phantom{-}0   &0   &0  &0\\
   \phantom{-}0   &\phantom{-}0    &0    &0   &{ 0.0058}   &0    &\phantom{-}{ 0.0410}   &0    &{ 0.0032}    &0\\
   {\red-0.4975}    &\phantom{-}{ 0.2804}    &{ 0.0092}    &0    &0    &{ 0.0939}    &\phantom{-}{ 0.1440}   &0  &0   &0\\
   \phantom{-}0   &\phantom{-}0   &{ 0.0008}    &{ 0.0673}   &0    &0    &\phantom{-}0   &0    &{ 0.0010}   &{ 0.0009}\\
   \phantom{-}{ 0.1165}    &\phantom{-}0    &0    &0    &0    &{ 0.0402}  &\phantom{-}0  &{ 0.0888}    &0    &{ 0.0046}\\
   \phantom{-}{ 0.0176}   &\phantom{-}{ 0.0182}    &0  &0    &0    &0   &\phantom{-}0   &{ 0.0134}    &{ 0.0007}   &0  
\end{bmatrix}
}
\end{equation}
which is depicted in the right subplot of Figure~\ref{fig:2d} with the same convention.
\begin{figure}[htb!]
     \centering
     \begin{subfigure}[b]{0.44\textwidth}
         \centering
         \includegraphics[width=\textwidth]{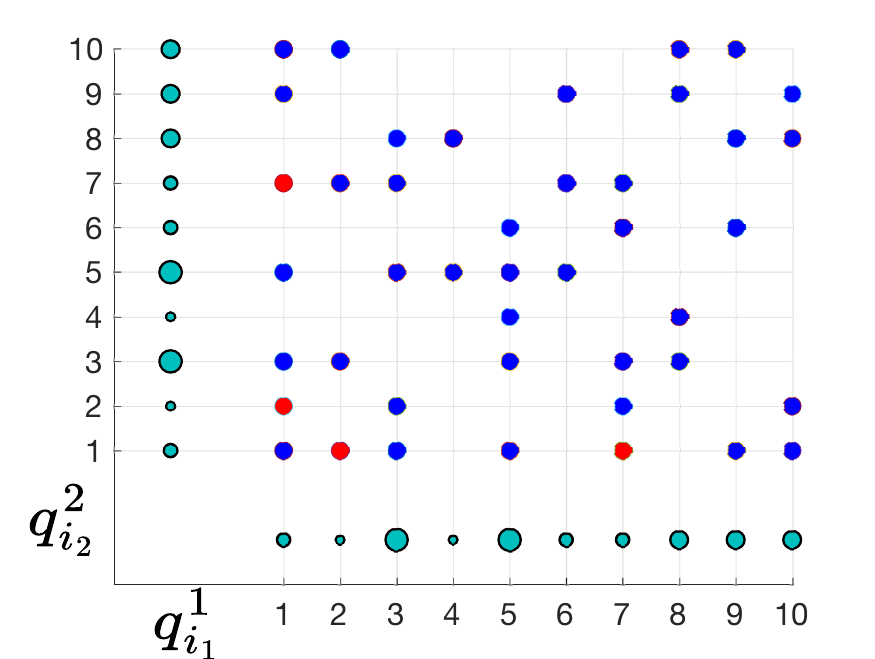}
     \end{subfigure}
     \begin{subfigure}[b]{0.45\textwidth}
         \centering
         \includegraphics[width=\textwidth]{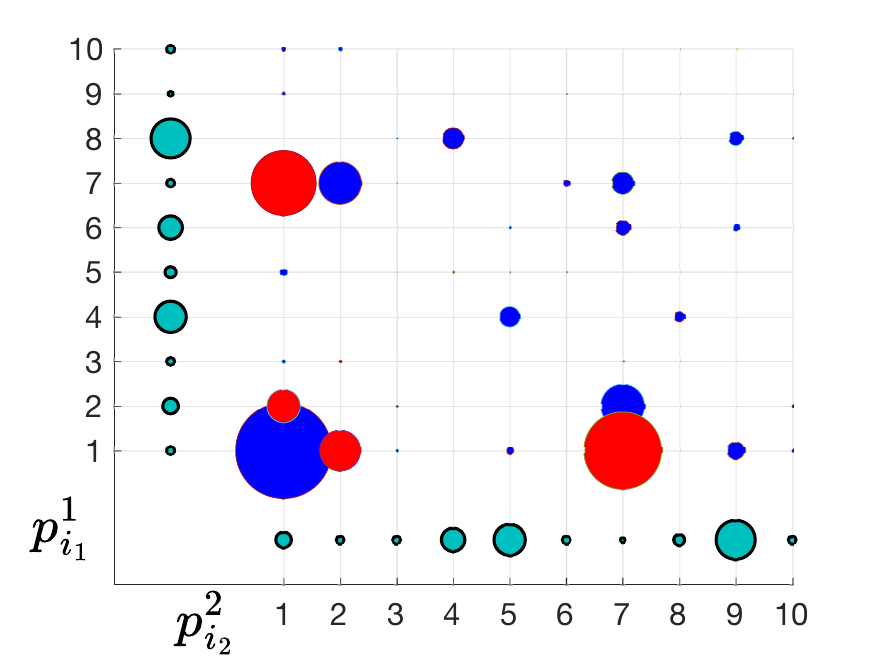}
     \end{subfigure}
     \caption{Sign-indefinite prior (left) and posterior (right) that matches specified marginals.}
     \label{fig:2d}
\end{figure}

\subsection{Three-marginal cases}
\subsubsection{Example 1}\label{ex:ex1} For a $3$-dimensional we select $Q_{i_1i_2i_3}$, {\footnotesize
\begin{align*}
Q(:,:,1)&=  \frac{1}{14}
    \begin{bmatrix}
    \phantom{-}1     &\phantom{-}1    &-1\\
    \phantom{-}1     &\phantom{-}0    &\phantom{-}1\\
    \phantom{-}1     &\phantom{-}0    &\phantom{-}1
    \end{bmatrix},\ \ 
Q(:,:,2)&=  \frac{1}{14}
    \begin{bmatrix}
    -1     &\phantom{-}0    &\phantom{-}0\\
    \phantom{-}1     &\phantom{-}1    &\phantom{-}0\\
    \phantom{-}1     &\phantom{-}1    &\phantom{-}0
    \end{bmatrix},\ \
Q(:,:,3)&=  \frac{1}{14}
    \begin{bmatrix}
     \phantom{-}1    &\phantom{-}1    &\phantom{-}1\\
     \phantom{-}1    &-1     &\phantom{-}1\\
    \phantom{-}0     &\phantom{-}1     &\phantom{-}1
    \end{bmatrix}.
\end{align*}}%
The ``prior marginals'' are
\begin{align*}
    q^{1}_{i_1} &:= \sum_{i_2i_3} Q_{i_1i_2i_3} = [0.2143,~0.3571,~0.4286],\\
    q^{2}_{i_2} &:= \sum_{i_1i_3} Q_{i_1i_2i_3} = [0.4286,~0.2857,~0.2857],\\
    q^{3}_{i_3} &:= \sum_{i_1i_2} Q_{i_1i_2i_3} = [0.3571,~0.2143,~0.4286].
\end{align*}
We now specify marginals
\begin{align*}
    p^1_{i_1} = [0.2,0.3,0.5], \ p^2_{i_2} = [0.4,0.4,0.2], \
    p^3_{i_3} = [0.1,0.6,0.3],
\end{align*}
and obtain the following values for the posterior $P^\star_{i_1i_2i_3}$, using Algorithm \ref{alg:gen-sinkhorn}, 
\begin{align*}
P^\star(:,:,1)&=  
    \begin{bmatrix}
    \phantom{-}0.1167   &\phantom{-}0.2368   &-0.5447\\
    \phantom{-}0.0659   &\phantom{-}0   &\phantom{-}0.1402\\
    \phantom{-}0.0272   &\phantom{-}0   &\phantom{-}0.0579
    \end{bmatrix},
P^\star(:,:,2)= 
    \begin{bmatrix}
-0.3456  &\phantom{-}0  &\phantom{-}0\\
\phantom{-}0.2210  &\phantom{-}0.4482  &\phantom{-}0\\
\phantom{-}0.0913  &\phantom{-}0.1851  &\phantom{-}0
    \end{bmatrix},\\
P^\star(:,:,3)&= 
    \begin{bmatrix}
    \phantom{-}0.1429   &\phantom{-}0.2898    &\phantom{-}0.3041\\
    \phantom{-}0.0806   &-0.8275   &\phantom{-}0.1716\\
    \phantom{-}0   &\phantom{-}0.0676    &\phantom{-}0.0709
    \end{bmatrix}.    
\end{align*}
Figure~\ref{fig:3d_small} displays the prior and posterior in three dimensions.
\begin{figure}[htb!]
     \centering
     \begin{subfigure}[b]{0.48\textwidth}
         \centering
         \includegraphics[width=\textwidth]{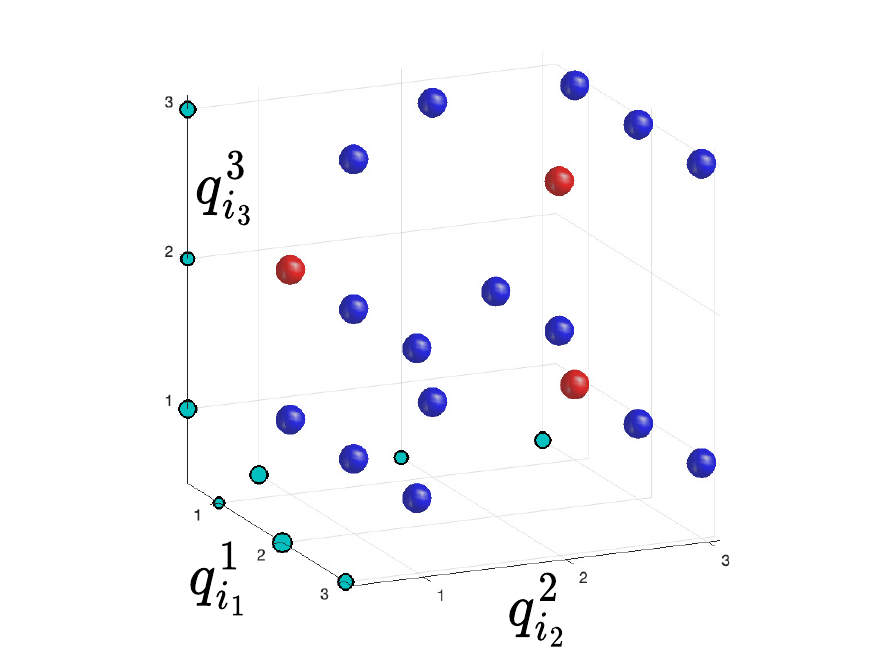}
     \end{subfigure}
     \begin{subfigure}[b]{0.48\textwidth}
         \centering
         \includegraphics[width=\textwidth]{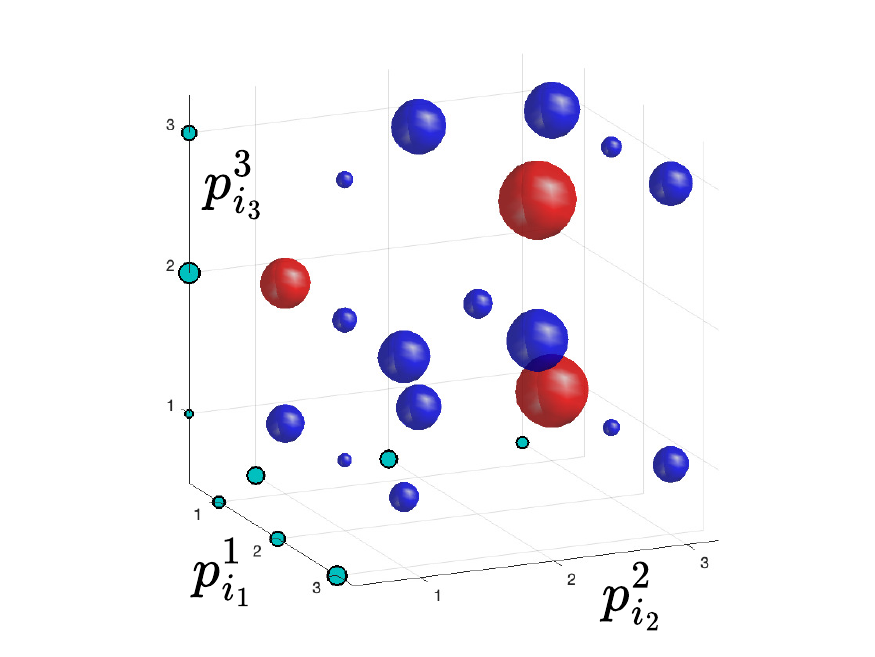}
     \end{subfigure}
     \caption{Example \ref{ex:ex1}: Sign-indefinite prior (left) and posterior (right).  }
     \label{fig:3d_small}
\end{figure}

\subsubsection{Example 2}\label{ex:ex2}
The generalized Sinkhorn algorithm has good scalability and linear convergence rate. To highlight this, we include a three-dimensional example where the prior has $2266$ positive entries and $51$ negative entries, randomly selected, and depicted in Figure~\ref{fig:3d_large}. 
\begin{figure}[htb!]
     \centering
     \begin{subfigure}[b]{0.48\textwidth}
         \centering
         \includegraphics[width=\textwidth,trim={1cm 0cm 1cm 1cm},clip]{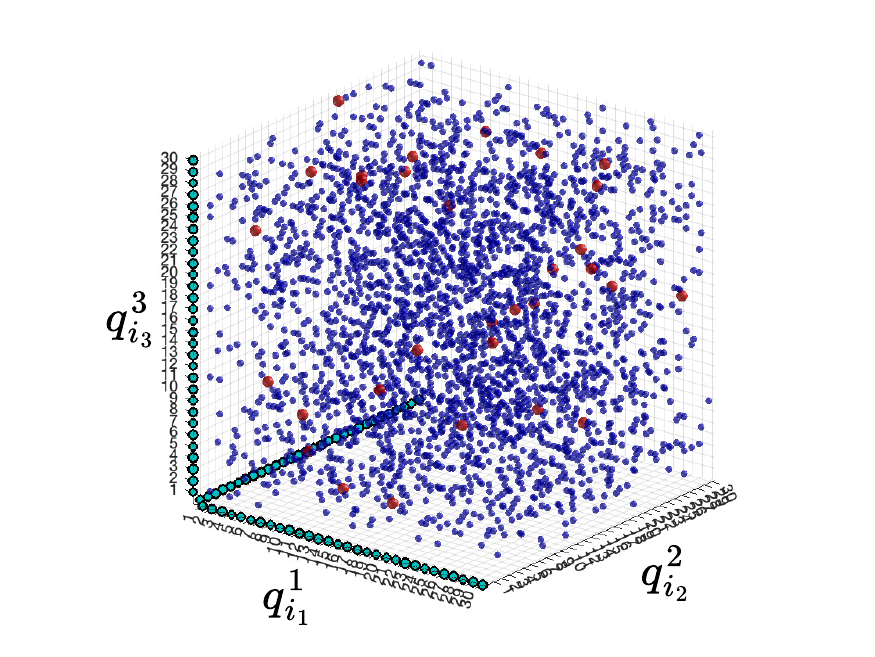}
     \end{subfigure}
     \begin{subfigure}[b]{0.48\textwidth}
         \centering
         \includegraphics[width=\textwidth,trim={1cm 0cm 1cm 1cm},clip]{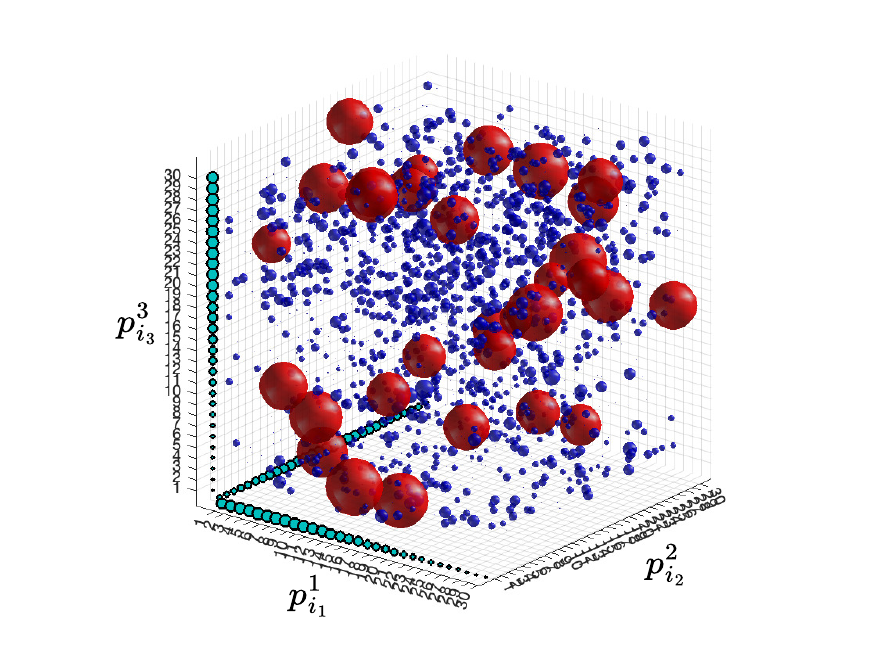}
     \end{subfigure}
     \caption{Large-scale example. Prior on left and posterior on right.}
     \label{fig:3d_large}
\end{figure}
Similarly, marginals are selected randomly, with code available in~\url{https://github.com/dytroshut/Generalized-Sinkhorn-Algorithm}, and the posterior is shown on the right subplot of Figure~\ref{fig:3d_large}.
The convergence of Algorithm \ref{alg:gen-sinkhorn} is displayed by plotting
the value of marginal constraint violation $\log(\| \sum_{i_1i_2 \cdots i_n \backslash i_\ell}  P^{*}_{i_1i_2\cdots i_n}-p_{i_\ell}^{\ell}\|)$ as a function of the iteration in Figure~\ref{fig:convergence_rate}.

\begin{figure}[!t]
\centering
    \includegraphics[width=.6\columnwidth,trim={0.3cm 0.5cm 1cm 0.5cm},clip]{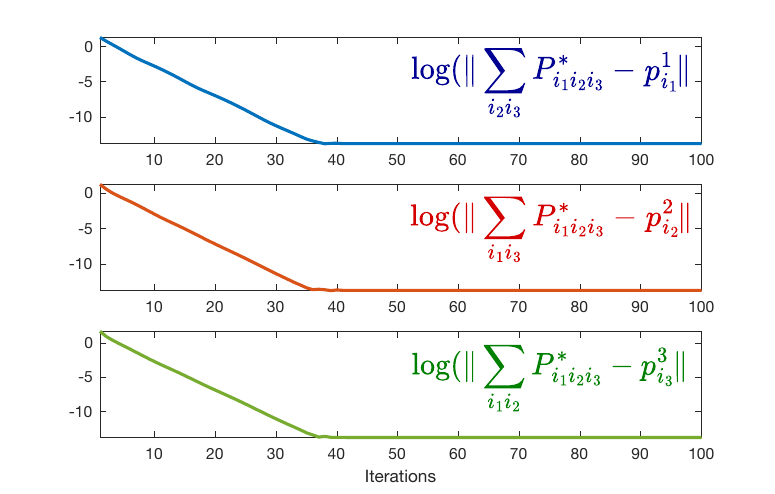}
    \caption{Marginal constraint violation as a function of iteration for the three marginals for Example \ref{ex:ex2}. The generalized Sinkhorn algorithm converges within $40$ iterations to machine precision.}
    \label{fig:convergence_rate}
\end{figure}

\section{Concluding remarks}\label{sec:conclusion}
The multi-index arrays $P,~Q$ can also be viewed as sign-indefinite probability laws, as advocated by Paul Dirac~\cite{dirac1942bakerian} and Richard Feynman~\cite{feynman1987negative}, to serve as a computational construct that models internal manifestation of underlying relations. To a degree, this was the starting point in our earlier work \cite{dong2023negative} aimed at gene regulatory networks. It is not uncommon in biological and also chemical networks~\cite[Chapter 5]{newman2018networks} that a sign-indefinite gain, that can be viewed as above, dictates concentration levels of proteins and other chemical substances. The precise measurement of such gains are of importance in understanding the structure and functionality of networks, in cancer-specific gene regulation,interaction between transcription factors, and the effects of various drug therapies~\cite{hopkins2008network}.
%sandhu2015graph,sandhu2016geometry,barretina2012cancer,cancer2015pharmacogenomic,ghandi2019next}.

% The proposed framework may holds a potential in the realm of comprehending biological networks and biomedical models~\cite[Chapter 5]{newman2018networks}. The introduction of a sign-indefinite prior opens up a new avenue for capturing and dissecting some of the most fundamental correlations and reactions between biological components. One application lies in cancer-specific gene regulatory networks, where the interactions between {\em transcription factors} (also, MicroRNAs) and genes associated with particular cancers are studied. Also, the tissue-specific gene regulatory networks investigate similar correlations within specified tissues. The networks mentioned above have been utilized to characterize human cancer models and facilitate anticancer drug identification~\cite {barretina2012cancer,cancer2015pharmacogenomic,ghandi2019next}.

Finally, we point out that a natural generalization of the basic problem in this paper, that can be treated in a similar manner, is to include a ``transportation cost'' $\langle C,P\rangle:=\sum_{i_1,\ldots,i_n}C_{i_1,\ldots,i_n}P_{i_1,\ldots,i_n}$, for a suitable tensor $C$ that penalizes coupling for combination of indices, in a problem to minimize $\langle C,P\rangle+S(P,Q)$ over posteriors $P$ that meet marginal constraints. 

\bibliographystyle{siamplain}
\bibliography{references}

\end{document}